\documentclass[lettersize,journal]{IEEEtran}
\usepackage{amsmath,amsfonts}
\usepackage{algcompatible}
\usepackage{algorithm}
\usepackage{array}
\usepackage[hidelinks]{hyperref}

\newcommand{\num}{m}

\usepackage{stackengine}
\newcommand\barbelow[1]{\stackunder[1.2pt]{$#1$}{\rule{.8ex}{.075ex}}}
\usepackage{amssymb}
\usepackage{textcomp}
\usepackage{stfloats}
\usepackage{url}
\usepackage{verbatim}
\usepackage{graphicx}
\usepackage{xcolor}
\usepackage{cite}
\usepackage{amsmath}
\usepackage{tabularx}
\usepackage[english]{babel}
\usepackage{amsthm}
\usepackage{multirow}
\usepackage{booktabs}
\usepackage{mathtools}
\usepackage{subcaption}
\usepackage{caption}

\newtheorem{definition}{\textbf{Definition}}

\newtheorem{assumption}{\textbf{Assumption}}
\newtheorem{theorem}{\textbf{Theorem}}

\def\BibTeX{{\rm B\kern-.05em{\sc i\kern-.025em b}\kern-.08em
    T\kern-.1667em\lower.7ex\hbox{E}\kern-.125emX}}

\begin{document}

\title{
Scalable and Reliable State-Aware Inference of High-Impact N-k Contingencies
\vspace{-1mm}
}

\author{
Lihao~Mai,~\IEEEmembership{Student Member,~IEEE,}
Chenhan Xiao,~\IEEEmembership{Student Member,~IEEE,}
and~Yang~Weng,~\IEEEmembership{Senior Member,~IEEE}
\thanks{
L. Mai, C. Xiao, and Y. Weng are with the School of Electrical, Computer and Energy Engineering at Arizona State University, Tempe, AZ 85281 USA, e-mail: \{lmai7, cxiao20, yang.weng\}@asu.edu.}
\vspace{-10mm}
}

\maketitle

\begin{abstract}
Increasing penetration of inverter-based resources, flexible loads, and rapidly changing operating conditions make higher-order $N\!-\!k$ contingency assessment increasingly important but computationally prohibitive. Exhaustive evaluation of all outage combinations using AC power-flow or ACOPF is infeasible in routine operation. This fact forces operators to rely on heuristic screening methods whose ability to consistently retain all critical contingencies is not formally established. 
This paper proposes a scalable, state-aware contingency inference framework designed to directly generate high-impact $N\!-\!k$ outage scenarios without enumerating the combinatorial contingency space. The framework employs a conditional diffusion model to produce candidate contingencies tailored to the current operating state, while a topology-aware graph neural network trained only on base and $N\!-\!1$ cases efficiently constructs high-risk training samples offline. Finally, the framework is developed to provide controllable coverage guarantees for severe contingencies, allowing operators to explicitly manage the risk of missing critical events under limited AC power-flow evaluation budgets.
Experiments on IEEE benchmark systems show that, for a given evaluation budget, the proposed approach consistently evaluates higher-severity contingencies than uniform sampling. This allows critical outages to be identified more reliably with reduced computational effort.
\end{abstract}

\vspace{-3mm}


\vspace{-3mm}

\section{Introduction}

The electric power grid is undergoing a structural transformation driven by distributed generation, large-scale integration of inverter-based resources, electrified loads, and emerging large flexible loads such as data centers \cite{villa2025advancements}. For example, the deployment of large data centers introduces concentrated load patterns that can locally increase network stress and reduce operational margins, making system security more sensitive to simultaneous outages of multiple components. More broadly, grid operating states evolve more frequently and network topology is routinely modified through switching operations, maintenance, and mitigation actions \cite{li2024artificial, numan2023role}. Also, load patterns vary across space and time, and renewable injections fluctuate on sub-hour timescales \cite{li2016real,kenyon2022renewables}. Under such conditions, contingency analysis is no longer a static planning exercise but must be performed repeatedly and quickly to ensure secure operation under changing system states \cite{che2017screening}. Although the long-standing $N\!-\!1$ criterion remains the primary operational standard, growing operational complexity and tighter operating margins have increased the need to assess higher-order $N\!-\!k$ contingencies alongside traditional $N\!-\!1$ analysis \cite{chen2005highrisk}. This need becomes particularly pronounced during stressed or rapidly changing operating conditions \cite{che2019hidden}. In practice, however, exhaustive $N\!-\!k$ analysis is rarely performed regularly because the number of $k$-outage combinations becomes infeasible for routine ACOPF evaluation as $k$ increases, and each scenario may require an expensive AC power-flow or ACOPF solve \cite{che2017screening}. This leads to a computational burden that makes routine AC power-flow/ACOPF evaluation infeasible, as illustrated in Fig.~\ref{fig:bigpic}.

\begin{figure*}
     \centering
    \includegraphics[width=1\linewidth]{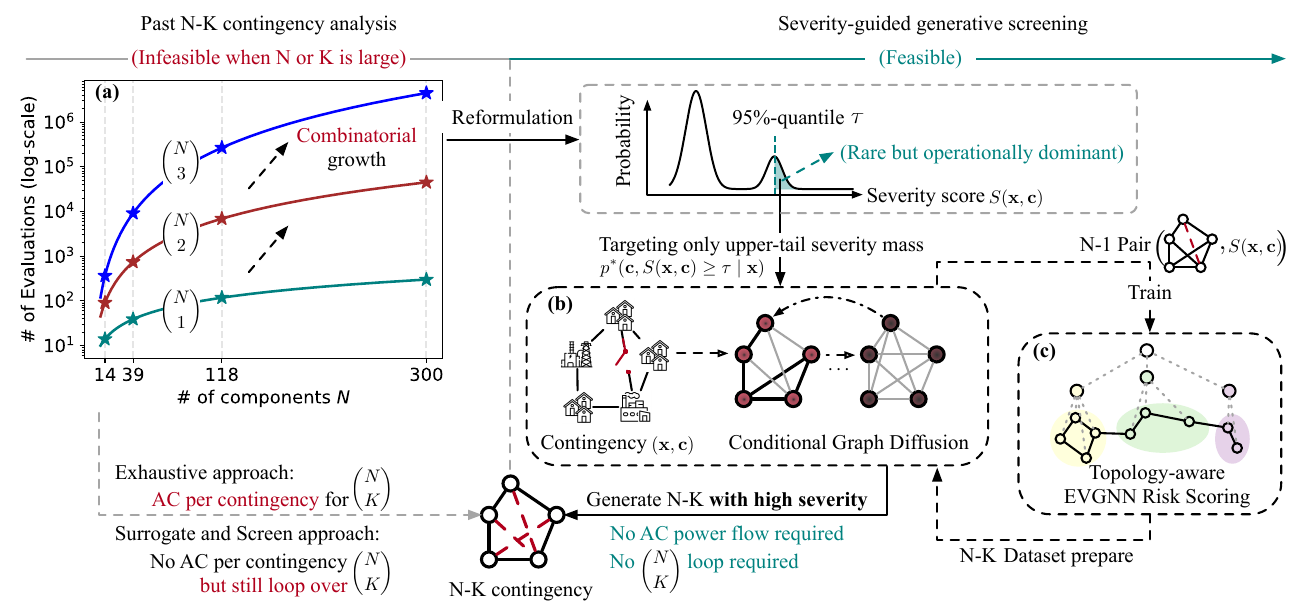}
    \caption{
Overview of the proposed sample-efficient $N\!-\!k$ contingency screening framework via conditional diffusion.
(a) In large power networks, exhaustive $N\!-\!k$ contingency analysis is computationally intractable due to the combinatorial growth of outage combinations, rendering brute-force AC power-flow evaluation infeasible.
(b) Rather than enumerating all contingencies, a conditional graph diffusion model directly samples outage patterns from the upper tail of the severity distribution, generating state-conditioned, high-risk $N\!-\!k$ contingency scenarios without looping over the full combinatorial space.
(c) A topology-aware EVGNN, trained once using only base-case and $N\!-\!1$ data, is reused as a fast risk surrogate to score generated $N\!-\!k$ contingency scenarios without additional AC power-flow simulations, enabling the construction of a compact, high-risk dataset through graph-based generalization from single- to multi-line outages.
    }
    \label{fig:bigpic}
\end{figure*}

To improve efficiency beyond exhaustive enumeration, existing research on $N\!-\!k$ contingency analysis can be broadly grouped into three classes \cite{yang2020power}. The first replaces repeated AC power-flow calculations with approximation or reduced-order models that estimate post-contingency system responses such as line flows, voltages, or security margins \cite{ronellenfitsch2016dual,wang2025progress,de2024role}. These methods reduce the cost per contingency but still require evaluating a large number of contingency scenarios, and their approximations may lose accuracy under stressed nonlinear operating conditions where only a small fraction of contingencies are truly critical.  
The second class adopts screening strategies based on linearized sensitivity metrics, such as line outage distribution factors (LODFs) and multi-outage interaction indices, to rank contingencies prior to detailed analysis \cite{guler2007glodf,davis2009multioutage,chen2015flexible,narimani2020graph,heidari2022accurate}. Although effective in reducing the number of detailed evaluations, such approaches still require screening a large number of $k$-outage combinations \cite{turitsyn2012n2,che2019hidden}, which is computationally costly for large $k$ in large systems.  

The third class treats contingency assessment as a supervised severity-prediction task: data-driven models estimate a contingency's severity score and flag scenarios for detailed AC power-flow verification \cite{zhao2017learning}. However, because severe contingencies are extremely rare in operational datasets, model calibration is difficult, and such models do not directly identify which contingencies will cause violations when they occur \cite{mansour1997dynamic}.
Despite their different directions, these approaches share two practical limitations. First, even accelerated methods remain computationally demanding because they still rely on evaluating or ranking large contingency sets. Second, fast screening strategies typically rely on heuristic ranking and do not provide guarantees that the most critical contingencies are retained, which is particularly problematic under stressed nonlinear operating conditions where only a small fraction of contingencies cause severe impacts. Consequently, operators face a tradeoff between computational efficiency and reliable identification of high-impact events.  

This limitation motivates a different approach: rather than enumerating all $N\!-\!k$ outage combinations, we aim to directly construct a small set of high-impact contingencies while providing a formal coverage guarantee for detailed AC power-flow/ACOPF verification under the current operating state within operational time constraints \cite{chen2024artificial}.
In other words, the goal is to replace combinatorial contingency enumeration with scalable, reliability-aware inference. This enables operators to evaluate far fewer contingencies while maintaining security coverage. Moreover, the proposed framework allows operators to control the probability of missing severe contingencies through a user-specified sampling budget, directly linking computational effort with operational risk tolerance.

To address this challenge, this paper reformulates efficient $N\!-\!k$ contingency assessment as a state-conditioned contingency generation problem rather than an enumeration or screening task.
The key observation is that identifying severe contingencies does not require explicitly enumerating all $\binom{N}{k}$ outage combinations \cite{ten2015cyber, fliscounakis2013contingency}. Instead, contingency configurations and their severity responses are highly concentrated in a small high-severity subset, with only a few combinations producing extreme system impacts \cite{tan2019deep, kamwa2002time}. Our objective is therefore not to evaluate or rank all contingencies exhaustively, but to identify the severe-contingency subset conditioned on the current operating state with a probabilistic coverage guarantee. The goal is that the generated shortlist of candidate contingencies can be forwarded for detailed AC power-flow/ACOPF verification. This enables operators to focus limited AC power-flow evaluations on system-critical scenarios rather than spending resources on mostly low-impact outage combinations.

To accurately characterize this severe-contingency subset while maintaining stable model training, we redesign a diffusion-model-based contingency generation framework for contingency analysis. The diffusion-model framework enables likelihood-based training and iteratively refines generated contingency scenarios toward regions of higher probability mass under the conditioned contingency distribution \cite{mai2025guaranteed, graikos2022diffusion}. This property is particularly suitable for contingency analysis, where severe scenarios are rare and an accurate representation of rare, high-severity contingencies is more important than uniform coverage of all scenarios \cite{sun2004visualizations}. Moreover, the generation procedure allows operators to balance computational burden against the capture of severe contingencies. We further show that the proposed approach provides a practical performance guarantee in which the top-$m$ contingencies reliably include the most severe events.

A practical challenge in implementing diffusion-based contingency generation lies in assembling labeled datasets for higher-order $N\!-\!k$ events \cite{yang2020power}. While the pre-contingency operating point and $N\!-\!1$ studies can be generated at moderate cost, the number of possible $N\!-\!k$ contingencies scales as $\binom{N}{k}$ and thus grows rapidly with $k$. This makes complete labeling infeasible \cite{chen2014contingency}. To mitigate this problem, we introduce a topology-aware edge varying graph neural network (EVGNN)-based risk predictor trained only on the pre-contingency operating point and $N\!-\!1$ cases. Built on the bus–branch graph, this model captures how stress propagates through network connectivity, enabling rapid severity scoring of multi-line outage patterns. In the offline workflow, large pools of contingency scenarios are generated and screened using this predictor. And, only top-ranked scenarios are retained to form a reduced high-risk contingency set for diffusion-model training. The trained diffusion model then produces sampled severe $N\!-\!k$ contingencies given the current operating state, without requiring AC power-flow evaluation for every combination.

We assess performance on IEEE 14-, 39-, 57-, and 118-bus benchmark systems by comparing against a uniform random sampling baseline under an equal ACPF evaluation budget per operating point. For each operating point, generated $N\!-\!k$ outage patterns are evaluated using ACPF to obtain the severity index defined later. We show (i) ACPF solve success rate, (ii) the fraction of contingencies whose severity falls in a high-severity band, and (iii) Top-$\num$ screening performance curves showing how quickly severe contingencies appear as evaluation budget increases. Numerical results illustrate that the proposed approach consistently focuses evaluations on high-impact contingencies and achieves higher Top-$m$ severities than uniform random sampling for the same computational budget.

The rest of the paper is organized as follows. Section~\ref{sec:problem} formulates the screening problem and associated risk objective. Section~\ref{sec:method} presents the EVGNN–diffusion pipeline. Section~\ref{sec:theorem} provides theoretical analysis. Section~\ref{sec:results} reports numerical results, and Section~\ref{sec:conclusion} concludes the paper.

\section{Problem Formulation}
\label{sec:problem}

In large-scale $N\!-\!k$ contingency analysis, two practical challenges arise. First, even faster methods than the past remain computationally demanding because they still require evaluating or screening a large number of contingencies. Second, many fast screening methods rely on heuristic rankings and cannot ensure that the most critical contingencies are kept. This issue becomes more serious under stressed operating conditions, where only a small fraction of contingencies cause severe impacts. As a result, operators face a tradeoff between computational efficiency and reliably identifying high-impact outages. The key challenge is therefore to efficiently search the large contingency space while ensuring that critical events are not missed. Motivated by this tradeoff, we consider the following task: for a given operating state $\mathbf{x}$, identify contingencies using a limited number of AC power-flow evaluations while ensuring that severe events are not missed. Missing a severe contingency can leave violations undetected and increase operational risk, so screening methods must control this risk rather than rely only on heuristic rankings. 

We define the system model, contingency representation, and severity metric used in this work. We consider a power system with $N$ components that may be subject to contingencies, such as transmission lines or transformers. 
Let $\mathbf{x}\in\mathcal{X}$ denote the pre-contingency system state, which may include network topology, load demands, renewable injections, generator setpoints, and other operational variables available to the system operator. A contingency is represented by a binary vector
\(
\mathbf{c}\in\{0,1\}^{N},
\)
where $c_i=1$ is that component $i$ is outaged and $c_i=0$ otherwise. 
An $N\!-\!k$ contingency corresponds to a $k$-element outage satisfying
\(
\|\mathbf{c}\|_{0}=k,
\)
i.e., exactly $k$ components are simultaneously removed from service. 
The total number of possible $N\!-\!k$ contingencies is therefore $\binom{N}{k}$, which grows combinatorially with both $N$ and $k$. 
Even when accelerated models are used, approaches that still enumerate or score this combinatorial set eventually become impractical for large systems. This motivates methods that avoid explicit enumeration altogether.

For each contingency, a severity metric is computed to evaluate its operational impact. 
Given a system state $\mathbf{x}$ and a contingency $\mathbf{c}$, the system response is assessed using standard contingency analysis tools (e.g., AC power-flow and/or ACOPF feasibility/violation checks), leading to a scalar severity index
\begin{equation}
s = S(\mathbf{x},\mathbf{c}) \in \mathbb{R}_{+}.
\end{equation}
The severity metric $S(\mathbf{x},\mathbf{c})$ serves as a scalar proxy for post-contingency operational stress as assessed by standard security analysis tools. 
In practice, $S(\mathbf{x},\mathbf{c})$ may aggregate criteria such as transmission thermal overloads, bus voltage magnitude violations, reactive power limit violations, or infeasibility and nonconvergence of the post-contingency AC power-flow solution. 
Larger values of $S(\mathbf{x},\mathbf{c})$ correspond to contingencies that impose greater operational burden and are more likely to require corrective actions, while $S(\mathbf{x},\mathbf{c})=0$ indicates a benign contingency with no security violations under the given operating state. 
Importantly, because the impact of a contingency depends strongly on the operating condition, an outage pattern that is harmless under light loading may become critical under stressed conditions, rendering static or precomputed contingency rankings unreliable.

For a fixed operating state $\mathbf{x}$, the mapping $\mathbf{c}\mapsto S(\mathbf{x},\mathbf{c})$ is typically highly skewed: most $N\!-\!k$ contingencies have little or no impact, while only a small subset produces severe impacts. To formalize this, for a high-severity threshold $\tau$, define the set of high-impact contingencies as
\begin{equation}
\mathcal{C}_{\tau}(\mathbf{x})
\triangleq
\left\{
\mathbf{c}\in\{0,1\}^{N}:\ \|\mathbf{c}\|_{0}=k,\ S(\mathbf{x},\mathbf{c})\ge \tau
\right\}.
\end{equation}
Because $\mathcal{C}_{\tau}(\mathbf{x})$ is typically small relative to the full set of $\binom{N}{k}$ contingency cases, the key practical challenge is to find severe cases without spending AC power-flow validation effort on the many low-impact scenarios. Given limited computational resources, let $B$ denote the available AC power-flow evaluation budget, i.e., the maximum number of contingencies that can be validated using AC power-flow/ACOPF checks for a given $\mathbf{x}$. 
Accordingly, the screening procedure outputs a screened list
\begin{equation}
\widehat{\mathcal{C}}(\mathbf{x}) \subseteq \{\mathbf{c}\in\{0,1\}^{N}:\ \|\mathbf{c}\|_{0}=k\},
\qquad
\left|\widehat{\mathcal{C}}(\mathbf{x})\right|\le B,
\end{equation}
which focuses AC power-flow validations on a small set of selected outages. To directly address the above tradeoff and align with our contributions, we target the following two properties:
\begin{itemize}
\item \textbf{Computational efficiency (faster screening):} 
Construct $\widehat{\mathcal{C}}(\mathbf{x})$ with $\left|\widehat{\mathcal{C}}(\mathbf{x})\right|\le B$ so that the limited AC power-flow validation budget is focused on system-critical contingencies, i.e., the evaluated contingencies $\mathbf{c}\in\widehat{\mathcal{C}}(\mathbf{x})$ should yield large values of $S(\mathbf{x},\mathbf{c})$ as often as possible, rather than expending AC power-flow evaluations on low-impact cases.

\item \textbf{Controllable coverage (performance guarantee):} 
Unlike heuristic screening methods that provide no assurance on retaining the most critical outages, provide a tunable, quantitative guarantee on how well the screened list captures the truly severe set $\mathcal{C}_{\tau}(\mathbf{x})$. 
For example, using the coverage metric
\begin{equation}
\mathrm{Cov}_{\tau}(\mathbf{x})
\triangleq
\frac{\left|\widehat{\mathcal{C}}(\mathbf{x}) \cap \mathcal{C}_{\tau}(\mathbf{x})\right|}{\left|\mathcal{C}_{\tau}(\mathbf{x})\right|},
\end{equation}
the method should allow the operator to select a desired coverage level (or probability of missing severe contingencies) and obtain a corresponding guarantee, thereby controlling the operational risk of overlooking rare but critical outage scenarios.
\end{itemize}

To achieve these two targets, we define the problem in this paper as:
\begin{itemize}
    \item \textbf{Problem:} Efficient identification of high-severity $N\!-\!k$ contingencies under a given system operating state.
    \item \textbf{Given:} 
    A pre-contingency system state $\mathbf{x}$, a contingency representation $\mathbf{c} \in \{0,1\}^N$ with $\|\mathbf{c}\|_0 = k$, and a severity function $s = S(\mathbf{x}, \mathbf{c})$ evaluated through standard security assessment models.
    \item \textbf{Find:} 
    A state-conditioned contingency scenario generative model that approximates the distribution
    \begin{equation}\label{eq:target-distribution}
        p^{\ast}(\mathbf{c}, S(\mathbf{x}, \mathbf{c}) \ge \tau \mid \mathbf{x}),
    \end{equation}
    where $\tau$ denotes a high-severity threshold, enabling direct generation of a small set of high-impact contingency vectors $\mathbf{c}$ without exhaustive enumeration of $\binom{N}{k}$ possibilities.
\end{itemize}

\section{Methodology}
\label{sec:method}

Building on the formulation in Section~\ref{sec:problem}, our goal is to develop a risk-directed generative framework for identifying high-severity $N\!-\!k$ contingencies under a limited AC power-flow validation budget. Rather than enumerating or ranking a large candidate set, the proposed approach directly constructs a compact set of grid-feasible outage patterns that are most likely to induce severe post-contingency impacts under the current operating state.

Formally, let $\pi(\mathbf{c}\mid\mathbf{x})$ denote a state-conditioned sampling policy and let $B$ denote the available AC validation budget. We consider the operator-oriented objective
\begin{equation}
\begin{aligned}
\label{eq:opt_sampling}
& \max_{\pi(\cdot\mid\mathbf{x})}\;\;
\mathbb{P}_{\mathbf{c}^{(1:B)}\sim \pi}\!\left(\max_{1\le j\le B} S(\mathbf{x},\mathbf{c}^{(j)}) \ge \tau\right)
\\ 
& \text{s.t.}\quad
\mathbf{c}^{(j)} \in \mathcal{F}_k(\mathcal{G}),
\end{aligned}
\end{equation}
which seeks to maximize the probability of capturing at least one high-severity contingency while respecting grid feasibility and a fixed AC validation budget.

We construct a practical approximation to~\eqref{eq:opt_sampling} using three components:  
(1) a state-conditioned generator $p_\theta(\mathbf{c}\mid\mathbf{x})$ that proposes candidate contingencies;  
(2) a coverage-control rule that selects $B$ for a user-specified miss probability; and  
(3) an offline AC-labeled data pipeline, together with a topology-aware EVGNN, that focuses learning on the high-risk region of the contingency space.

\noindent\textbf{Why this is not ``another ranking model.''} Classical screening methods (e.g., LODF- or surrogate-based) still require evaluating or scoring a large portion of the $\binom{N}{k}$ candidate set before selecting a shortlist. In contrast, \eqref{eq:opt_sampling} formulates the task as risk-constrained contingency inference: directly generate a budget-sized set of grid-feasible $N\!-\!k$ outage patterns that maximizes the probability of capturing at least one truly severe event. The diffusion sampler is used as an efficient numerical mechanism to approximately solve \eqref{eq:opt_sampling} online without looping over the combinatorial contingency space.

\subsection{Fast state-conditioned generation of high-impact $N\!-\!k$ contingencies}
\label{subsec:diffusion}

State dependence is essential in practice: the same line-pair outage may have little impact under light loading but can trigger thermal or voltage violations, or even nonconvergence, under stressed conditions. As a result, static prioritization may miss critical events when operating margins tighten.
Let $\mathbf{c}_0\in\{0,1\}^{N}$ denote a contingency vector with $\|\mathbf{c}_0\|_0=k$. We generate outage patterns conditioned on operating state $\mathbf{x}$ using a diffusion-style denoising process. The forward process perturbs $\mathbf{c}_0$ by adding Gaussian noise over $T$ steps:
\begin{equation}
q(\mathbf{c}_t \mid \mathbf{c}_{t-1})
=
\mathcal{N}\!\big(\sqrt{1-\beta_t}\,\mathbf{c}_{t-1},\, \beta_t \mathbf{I}\big),
\end{equation}
and equivalently
\begin{equation}
\mathbf{c}_t
=
\sqrt{\bar{\alpha}_t}\,\mathbf{c}_0
+
\sqrt{1-\bar{\alpha}_t}\,\boldsymbol{\epsilon},
\qquad
\boldsymbol{\epsilon}\sim\mathcal{N}(\mathbf{0},\mathbf{I}),
\end{equation}
where $\bar{\alpha}_t=\prod_{i=1}^{t}(1-\beta_i)$. The reverse process removes noise to recover a contingency pattern from $\mathbf{c}_T$, and the reverse transition is parameterized by $\boldsymbol{\epsilon}_\theta(\mathbf{c}_t,\mathbf{x},t)$ conditioned on state $\mathbf{x}$:
\begin{equation}
p_\theta(\mathbf{c}_{t-1} \mid \mathbf{c}_t, \mathbf{x})
=
\mathcal{N}\!\Big(
\frac{1}{\sqrt{1-\beta_t}}
\big(\mathbf{c}_t - \beta_t \boldsymbol{\epsilon}_\theta(\mathbf{c}_t, \mathbf{x}, t)\big),
\tilde{\beta}_t \mathbf{I}
\Big).
\end{equation}

Severe contingencies are rare but operationally dominant. To emphasize learning on high-impact events, we weight the denoising loss using severity $s=S(\mathbf{x},\mathbf{c}_0)$:
\begin{equation}
\mathcal{L}_{\mathrm{diff}}
=
\mathbb{E}_{\mathbf{c}_0,\,\mathbf{x},\,t,\,\boldsymbol{\epsilon}}
\!\left[
w(s)\,
\big\|
\boldsymbol{\epsilon}
-
\boldsymbol{\epsilon}_\theta(\mathbf{c}_t, \mathbf{x}, t)
\big\|_2^2
\right],
\end{equation}
where $w(s)$ is nondecreasing in $s$ (e.g., larger weight when $s\ge\tau$), thereby prioritizing contingencies that drive operational violations.
The reverse process yields $\widehat{\mathbf{c}}_0\in\mathbb{R}^N$. We enforce the $k$-outage structure and grid feasibility constraints using the projection operator
\begin{equation}
\mathbf{c}
=
\Pi_{\mathcal{F}_k(\mathcal{G})}(\widehat{\mathbf{c}}_0),
\label{eq:proj}
\end{equation}
which sets the $k$ largest entries to one, the rest to zero, and applies engineering feasibility rules such as excluding non-contingencable elements or islanding patterns. Equation~\eqref{eq:proj} ensures the generator produces grid-feasible patterns rather than unconstrained combinations.
Standard diffusion sampling reproduces the training distribution. To better approximate the risk-directed objective \eqref{eq:opt_sampling}, we modify the reverse sampling step to steer candidates toward higher predicted severity using a differentiable risk estimator $\hat S_\phi$ (EVGNN):
\begin{equation}
\label{eq:risk_guided}
\mathbf{c}_{t-1}
=
\widetilde{\mathbf{c}}_{t-1}
+
\lambda \nabla_{\mathbf{c}} \hat S_\phi(\mathbf{x},\mathbf{c}_t),
\end{equation}
where $\widetilde{\mathbf{c}}_{t-1}$ is the mean of the Gaussian reverse step and $\lambda\ge 0$ controls guidance strength. This turns generation into risk-directed contingency inference rather than generic sampling. Unlike LODF-based methods that scan many candidate combinations, the proposed sampler directly generates a small candidate set under the current state $\mathbf{x}$, leaving only this shortlist for AC validation.
Multi-outage impacts are not additive due to network interactions. To capture this structure, severity is represented as
\begin{equation}
\label{eq:interaction}
S(\mathbf{x},\mathbf{c})
=
\sum_i c_i r_i(\mathbf{x})
+
\sum_{i<j} c_i c_j I_{ij}(\mathbf{x}),
\end{equation}
where $r_i(\mathbf{x})$ represents single-outage effects and $I_{ij}(\mathbf{x})$ represents interaction effects. This explains why $N\!-\!1$ knowledge alone is insufficient and why state-conditioned inference must capture interaction terms.
At inference time, we generate a compact set of grid-feasible $N\!-\!k$ contingencies for the observed state $\mathbf{x}$ and validate only these using AC power-flow or ACOPF checks. This shifts computation away from broad enumeration toward targeted validation of operationally relevant outage scenarios.

\subsection{Coverage-controlled screening with performance guarantees}
\label{subsec:guarantee}

Heuristic screening provides no guarantee that severe contingencies remain in the screened set.
To provide controllable coverage, we explicitly connect the screened list size to an operator-specified probability of missing severe contingencies.
Fix a severity threshold $\tau$.
For a given operating state $\mathbf{x}$ and generator distribution $p_\theta(\mathbf{c}\mid\mathbf{x})$, define the severe-event capture probability
\begin{equation}
p_\tau(\mathbf{x})
:=
\mathbb{P}_{\mathbf{c}\sim p_\theta(\cdot\mid\mathbf{x})}\!\left[S(\mathbf{x},\mathbf{c})\ge\tau\right].
\end{equation}
If we generate and AC-validate $B$ independent contingencies $\{\mathbf{c}^{(j)}\}_{j=1}^{B}$, then
\begin{equation}
\mathbb{P}\!\left(\max_{1\le j\le B} S(\mathbf{x},\mathbf{c}^{(j)}) < \tau \right)
=
\left(1-p_\tau(\mathbf{x})\right)^B,
\end{equation}
and therefore
\begin{equation}
\mathbb{P}\!\left(\exists\,j \in \{1,\dots,B\} \text{ s.t. } S(\mathbf{x},\mathbf{c}^{(j)})\ge\tau\right)
=
1-\left(1-p_\tau(\mathbf{x})\right)^B.
\end{equation}

In operation, $p_\tau(\mathbf{x})$ is unknown.
We estimate it offline using AC validations of generated contingencies over a benchmark set of operating states and compute a conservative lower bound $p_{\tau}^{\mathrm{L}}$ (e.g., a binomial lower confidence bound).
Using $p_{\tau}^{\mathrm{L}}$, we choose $B$ to meet an operator-selected probability of missing severe contingencies $\delta_{\mathrm{miss}}$:
\begin{equation}
\left(1-p_{\tau}^{\mathrm{L}}\right)^B \le \delta_{\mathrm{miss}}
\quad \Longleftrightarrow \quad
B \ge \frac{\log(\delta_{\mathrm{miss}})}{\log\!\left(1-p_{\tau}^{\mathrm{L}}\right)}.
\label{17}
\end{equation}
This provides an explicit operational knob: smaller $\delta_{\mathrm{miss}}$ implies larger $B$ (more AC power-flow validation) and higher reliability; larger $\delta_{\mathrm{miss}}$ reduces computation.

\noindent\textbf{Operational interpretation.}
Given a chosen threshold $\tau$ and miss tolerance $\delta_{\mathrm{miss}}$, \eqref{17} sets the minimum number of AC power-flow validations $B$ required so that the probability of missing all contingencies with $S(\mathbf{x},\mathbf{c})\ge\tau$ is at most $\delta_{\mathrm{miss}}$. This holds provided the offline benchmark yields a conservative lower bound $p^L_\tau$ applicable to the operating regime of interest.
This makes the reliability--computation tradeoff explicit and operator-selectable, rather than implicit in a heuristic ranking rule.

\subsection{Offline data preparation and high-risk screening for training}
\label{subsec:data}

The diffusion generator in Section~\ref{subsec:diffusion} requires training data reflecting AC security impacts, but exhaustive $N\!-\!k$ labeling is intractable. To support both fast generation and reliable coverage benchmarking, we construct datasets in two stages: (i) AC-labeled base and $N\!-\!1$ studies, and (ii) compact high-risk multi-outage data identified by a fast network-based risk estimator and labeled selectively using AC power-flow.

We first build an offline dataset on a bus--branch model by sampling feasible operating states and evaluating the base case and all $N\!-\!1$ contingencies using AC power-flow. Each operating state is represented by a feature vector $\mathbf{x}_t$, obtained by perturbing loads and generator setpoints around nominal values. For each $\mathbf{x}_t$, we solve the base-case AC power-flow and discard infeasible states, then evaluate all single outages. The base/$N\!-\!1$ set is
\(
\mathcal{C}_{N-1} := \{\mathbf{0}\}\cup\{\mathbf{e}_i: i=1,\dots,N\}.
\)
Using AC results, each $(\mathbf{x}_t,\mathbf{c})$ receives severity
\(
s = S(\mathbf{x}_t,\mathbf{c}) \in \mathbb{R}_{\ge 0}.
\)
A representative metric is
\begin{equation}\label{eq:severity}
S(\mathbf{x}_t,\mathbf{c})
=
\max_{e}|P^{\mathrm{post}}_{t,e}(\mathbf{c}) - P^{\mathrm{base}}_{t,e}|
+
\max_{i}|V^{\mathrm{post}}_{t,i}(\mathbf{c}) - 1|.
\end{equation}

If the post-contingency AC power-flow fails to converge, we assign a large sentinel severity value
\(
S(\mathbf{x}_t,\mathbf{c}) = S_{\mathrm{fail}},
\)
treating nonconvergence as a severe operational outcome. Collecting all states and contingencies yields
\begin{equation}\label{eq:training-N-1}
\mathcal{D}_{N-1}
=
\Big\{
(\mathbf{x}_t,\mathbf{c}, s) :
t\in\mathcal{T}_{\mathrm{train}},
\mathbf{c}\in\mathcal{C}_{N-1},
s=S(\mathbf{x}_t,\mathbf{c})
\Big\}.
\end{equation}
To expose the generator to informative multi-outage patterns without enumerating $\binom{N}{k}$ cases, we introduce an auxiliary risk estimator $\hat S_{\phi}(\mathbf{x},\mathbf{c})$ trained using only $\mathcal{D}_{N-1}$:
\(
  \hat{s} = \hat{S}_\phi(\mathbf{x}, \mathbf{c}).
\)
For each operating state $\mathbf{x}_t$ and outage-count range $k\in[k_{\min},k_{\max}]$, we randomly sample multi-outage patterns and retain a compact high-risk set
\(
\mathcal{H}
=
\big\{ (\mathbf{x}_t, \mathbf{c}) \;:\; \hat{S}_\phi(\mathbf{x}_t, \mathbf{c}) \text{ is large} \big\}.
\)
We then selectively run AC power-flow on $\mathcal{H}$ to obtain labels $s=S(\mathbf{x}_t,\mathbf{c})$ for diffusion training and coverage benchmarking. This concentrates offline AC computation on operationally relevant outage patterns.
Online computation is dominated by the number of post-contingency AC power-flow or ACOPF runs. EVGNN scoring and diffusion sampling are lightweight forward passes whose cost is typically negligible compared with one AC validation, and the offline training cost is amortized over many operating hours. Thus, the practical value of the method is measured by how many fewer AC validations are needed to achieve a target $\delta_{\mathrm{miss}}$ and severity threshold $\tau$.

\begin{algorithm}[t]
\caption{Risk-Directed Generative $N$--$k$ Contingency Screening Workflow}
\label{alg:diffusion_nk}
\begin{algorithmic}[1]
\REQUIRE 
Network model $\mathcal{G}$; AC power-flow/ACOPF solver; severity metric $S(\mathbf{x},\mathbf{c})$; outage order $k$; severity threshold $\tau$; miss tolerance $\delta_{\mathrm{miss}}$.

\STATE Offline (planning and engineering studies)
\FOR{each feasible operating state $\mathbf{x}_t$}
    \STATE Solve base-case AC power-flow
    \FOR{each single outage $\mathbf{c}\in\mathcal{C}_{N-1}$}
        \STATE Solve post-contingency AC power-flow and compute $s=S(\mathbf{x}_t,\mathbf{c})$
        \STATE Store $(\mathbf{x}_t,\mathbf{c},s)$ in $\mathcal{D}_{N-1}$
    \ENDFOR
\ENDFOR
\STATE Train risk estimator $\hat S_\phi(\mathbf{x},\mathbf{c})$ on $\mathcal{D}_{N-1}$

\FOR{each training state $\mathbf{x}_t$}
    \STATE Sample $k$-outage patterns $\mathbf{c}\in\mathcal{F}_k(\mathcal{G})$
    \STATE Score using $\hat S_\phi$ and retain high-risk set $\mathcal{H}(\mathbf{x}_t)$
    \STATE Run AC power-flow/ACOPF on $\mathcal{H}(\mathbf{x}_t)$ to obtain $S(\mathbf{x}_t,\mathbf{c})$
\ENDFOR
\STATE Train generator $p_\theta(\mathbf{c}\mid\mathbf{x})$ and estimate lower bound $p_\tau^{\mathrm{L}}$

\STATE Online (real-time operation)
\STATE Observe current operating state $\mathbf{x}$
\STATE Select validation budget $B$ s.t. $(1-p_\tau^{\mathrm{L}})^B \le \delta_{\mathrm{miss}}$
\STATE Generate $B$ contingencies using \eqref{eq:risk_guided} and \eqref{eq:proj}
\FOR{each generated contingency $\mathbf{c}^{(j)}$}
    \STATE Run AC power-flow/ACOPF and compute $S(\mathbf{x},\mathbf{c}^{(j)})$
\ENDFOR
\STATE Output validated high-severity contingencies
\end{algorithmic}
\end{algorithm}

\section{Theoretical Analysis}
\label{sec:theorem}
\subsection{Computational Complexity Analysis}
\label{subsec:complexity}

We compare the online computational complexity of representative $N\!-\!k$ contingency assessment paradigms using Big-$\mathcal{O}$ notation.
Let $T_{\mathrm{PF}}$ denote the cost of one AC power-flow (or AC-OPF) evaluation, $T_{\mathrm{ml}}$ the cost of a forward pass of a learned model, and $T_{\mathrm{gen}}$ the cost of generating one contingency sample from the diffusion model, with $T_{\mathrm{PF}} \gg T_{\mathrm{ml}}$ and $T_{\mathrm{gen}}=\mathcal{O}(T_{\mathrm{ml}})$ in practice.
Classical exhaustive $N\!-\!k$ analysis evaluates all contingencies via AC power-flow, yielding
\(
\mathcal{O}\!\left(\binom{N}{k} T_{\mathrm{PF}}\right).
\)
Surrogate- and classification-based methods reduce the per-contingency cost but still require scoring the full combinatorial space, resulting in
\(
\mathcal{O}\!\left(\binom{N}{k} T_{\mathrm{ml}}\right).
\)
Screening-and-ranking approaches restrict expensive AC evaluations to a small subset but must still score all contingencies, leading to
\(
\mathcal{O}\!\left(\binom{N}{k} T_{\mathrm{ml}} + \num T_{\mathrm{PF}}\right),
\)
where $\num \ll \binom{N}{k}$ denotes the number of screened scenarios.

In contrast, the proposed generative framework avoids explicit enumeration by directly sampling high-risk contingencies conditioned on the operating state.
By generating $\num$ candidate scenarios and validating only these cases via AC power-flow, the resulting online complexity is
\[
\mathcal{O}\!\left(\num T_{\mathrm{gen}} + \num T_{\mathrm{ml}} + \num T_{\mathrm{PF}}\right),
\qquad \num \ll \binom{N}{k}.
\]
This replaces the combinatorial dependence on $\binom{N}{k}$ with a tunable sampling budget, enabling scalable and risk-aware $N\!-\!k$ contingency assessment.
Table~\ref{tab:complexity} summarizes the corresponding complexity comparisons.
Offline training costs are amortized across operating conditions and do not affect online scalability.

\begin{table}[t]
\centering
\caption{Computational complexity comparison of $N\!-\!k$ contingency assessment methods.}
\label{tab:complexity}
\setlength{\tabcolsep}{2pt} 
\renewcommand{\arraystretch}{0.95}
\begin{tabular}{l c}
\toprule
\textbf{Method} & \textbf{Online Complexity} \\
\midrule
Exhaustive AC power-flow $N\!-\!k$ analysis
& $\mathcal{O}\!\left(\binom{N}{k} T_{\mathrm{PF}}\right)$ \\

Surrogate-based evaluation
& $\mathcal{O}\!\left(\binom{N}{k} T_{\mathrm{ml}}\right)$ \\

Classification-based assessment
& $\mathcal{O}\!\left(\binom{N}{k} T_{\mathrm{ml}}\right)$ \\

Screening + AC power-flow validation
& $\mathcal{O}\!\left(\binom{N}{k} T_{\mathrm{ml}} + \num T_{\mathrm{PF}}\right)$ \\

\midrule
\textbf{Proposed generative approach}
& $\mathcal{O}\!\left(\num T_{\mathrm{gen}} + \num T_{\mathrm{ml}} + \num T_{\mathrm{PF}}\right)$ \\
\bottomrule
\end{tabular}
\end{table}

\subsection{Sample-Efficient Coverage of High-Severity Contingencies}
\label{subsec:guarantee}

Rather than exhaustively enumerating all $N\!-\!k$ contingencies, the proposed framework relies on a generative diffusion model to directly generate a compact set of high-risk outage scenarios conditioned on the operating state.
Throughout this subsection, we use $\num$ to denote the total number of contingency scenarios generated by the diffusion model for a given operating point.
These $\num$ generated scenarios collectively form the retained candidate set for downstream security assessment.

A fundamental question is how many generated scenarios are required to reliably cover the most severe contingencies.
To this end, we provide a probabilistic coverage guarantee showing that a relatively small number of generated scenarios suffices to capture a prescribed upper quantile of the severity distribution, provided the learned generative model approximates the true contingency distribution with bounded divergence.

\begin{definition}[High-severity region]
\label{assump:severity-region}
Let $p^\star(\mathbf{c},s)$ denote the true joint distribution of contingencies $\mathbf{c}$ and severity index $s$.
Define the $(1-\delta)$-quantile severity threshold
\begin{equation}\label{eq:quantile}
s_\delta := \inf\{\barbelow{s}\in\mathbb{R} : \mathbb{P}_{p^\star}(s \ge \barbelow{s}) \le \delta\},
\end{equation}
and the corresponding high-severity region
\begin{equation}\label{eq:severity-region}
\mathcal{A}_\delta := \{(\mathbf{c},s)\in\mathcal{C}\times\mathbb{R} : s \ge s_\delta\}.
\end{equation}
\end{definition}

\begin{assumption}[Generative model fidelity]
\label{assump:kl-bound}
Let $p_\theta(\mathbf{c},s)$ denote the diffusion-based generative model trained to approximate $p^\star(\mathbf{c},s)$.
Assume that its divergence from the true distribution is bounded as
\(
\mathrm{KL}\!\left(p^\star(\mathbf{c},s)\ \|\ p_\theta(\mathbf{c},s)\right) \le \varepsilon
\)
for some $\varepsilon>0$.
\end{assumption}

Assumption \ref{assump:severity-region} does not require exact recovery of the true contingency–severity distribution, which is infeasible in the $N-k$ setting. Instead, the KL bound $\varepsilon$ represents aggregate modeling error due to finite data and approximation in the generative model, and yields a conservative measure of how much probability mass may be misplaced. Under this assumption, we can quantify how much mass the learned model assigns to the high-severity region, which directly leads to the coverage guarantee stated in Theorem \ref{thm:coverage}.

\begin{theorem}[Coverage of high-severity contingencies]
\label{thm:coverage}
Under Assumptions~\ref{assump:severity-region}, 
let $(\mathbf{c}_i,s_i)_{i=1}^{\num}$ be $\num$ i.i.d.\ samples drawn from $p_\theta$.
the learned generative model assigns probability mass
\(
p_\theta(\mathcal{A}_\delta) \ge \delta - \sqrt{\varepsilon/2}.
\)
Moreover, for any $\eta\in(0,1)$,
\begin{equation}\label{eq:main-thm}
\begin{aligned}
\mathbb{P}&\!\left(
\frac{1}{m}\sum_{i=1}^{m} \mathbf{1}\{s_i \ge s_\delta\}
\ge 1-\eta
\right)\\
&\ge 1 
 - \exp\!\left(
-\frac{\eta^2 m (\delta-\sqrt{\varepsilon/2})}{2}
\right).
\end{aligned}
\end{equation}
\end{theorem}

\begin{proof}
We first establish a lower bound on the probability mass that the learned model assigns to the high-severity region $\mathcal{A}_\delta$.
By Pinsker's inequality,
\(
\|p^\star-p_\theta\|_{\mathrm{TV}}
\le
\sqrt{\tfrac{1}{2}\mathrm{KL}(p^\star\|p_\theta)}
\le
\sqrt{\varepsilon/2}.
\)
For any measurable set $B$, the total variation distance bounds the probability discrepancy as
\(
|p^\star(B)-p_\theta(B)|\le \|p^\star-p_\theta\|_{\mathrm{TV}}.
\)
Applying this inequality to $B=\mathcal{A}_\delta$ and using $p^\star(\mathcal{A}_\delta)=\delta$ yields
\(
p_\theta(\mathcal{A}_\delta)
\ge
\delta-\sqrt{\varepsilon/2}.
\)
Now consider $\num$ i.i.d.\ samples $(\mathbf{c}_i,s_i)_{i=1}^{\num}$ drawn from $p_\theta$, and define
\(
N := \sum_{i=1}^{\num} \mathbf{1}\{s_i \ge s_\delta\}
\)
as the number of generated samples whose severity exceeds the $(1-\delta)$-quantile threshold.
Conditional on $p_\theta(\mathcal{A}_\delta)=q_\theta$, the random variable $N$ follows a $\mathrm{Binomial}(\num,q_\theta)$ distribution.
From the bound above, $q_\theta \ge q$, where $q := \delta-\sqrt{\varepsilon/2}$.
Since a binomial random variable is stochastically increasing in its success probability, it follows that for any $t \in \mathbb{R}$,
\(
\mathbb{P}(N < t)
\le
\mathbb{P}_{\mathrm{Bin}(\num,q)}(N < t).
\)
The event
\(
\left\{
\frac{1}{\num}\sum_{i=1}^{\num} \mathbf{1}\{s_i \ge s_\delta\}
\ge 1-\eta
\right\}
\)
is equivalent to $\{N \ge (1-\eta)\num\}$.
Applying the multiplicative Chernoff bound for a $\mathrm{Binomial}(\num,q)$ random variable yields
\(
\mathbb{P}(N < (1-\eta)\num)
\le
\exp\!\left(-\tfrac{\eta^2 \num q}{2}\right).
\)
Substituting $q=\delta-\sqrt{\varepsilon/2}$ completes the proof.
\end{proof}

Theorem~\ref{thm:coverage} shows that the number of generated contingency scenarios required to reliably cover the top-$\delta$ severity tail scales inversely with the effective tail mass $\delta-\sqrt{\varepsilon/2}$ and linearly with the desired recall level.
Crucially, this scaling is independent of the combinatorial size of the underlying $N\!-\!k$ contingency space.
Beyond its theoretical significance, the result provides a practical guideline for selecting the diffusion sampling budget $\num$.
Operationally, the probabilistic guarantee enables system operators to directly control the risk of missing critical contingencies by adjusting $\num$, thereby achieving reliable security assessment while limiting expensive AC power-flow evaluations.
Such risk-based guarantees align naturally with existing operational practices that already account for uncertainty in load, renewable generation, and contingency occurrence.
In this sense, the coverage guarantee establishes a principled link between model fidelity, risk tolerance, and computational budget, underpinning the scalability and tunability of the proposed generative screening framework.

\begin{figure*}[t]
    \centering
    
    \begin{subfigure}[t]{0.48\textwidth}
        \centering
        \includegraphics[width=\linewidth]{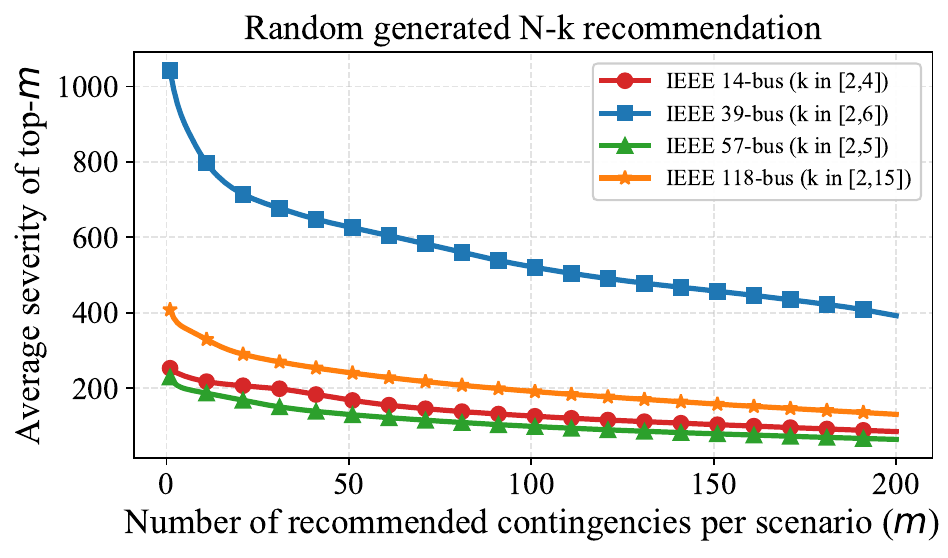}
        \caption{Random-generated $N-k$ recommendation case average severity comparison.}
        \label{fig:randomgenerated}
    \end{subfigure}
    \hfill
    \begin{subfigure}[t]{0.48\textwidth}
        \centering
        \includegraphics[width=\linewidth]{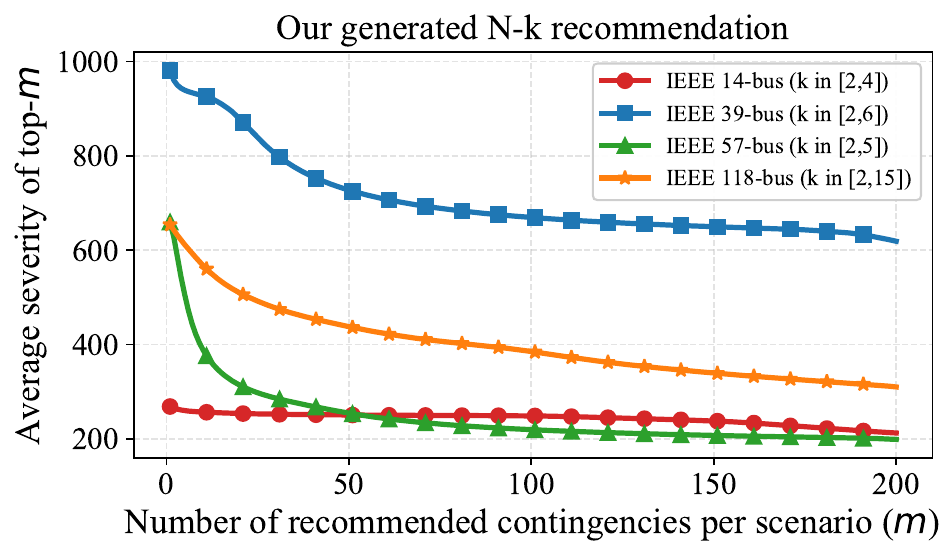}
        \caption{EVGNN-Diffusion-generated $N-k$ recommendation case average severity comparison.}
        \label{fig:diffusiongenerated}
    \end{subfigure}
    \caption{
    Recommendation curves for convergent $N-k$ contingencies across 200 operating scenarios on four IEEE benchmark systems. \textbf{Left:} uniform random sampling. \textbf{Right:} the proposed conditional diffusion generator. For each system and scenario, $\num$ contingencies are produced under the method of the corresponding panel within the specified $k$-range (IEEE 14: $k\in[2,4]$, IEEE 39: $k\in[2,6]$, IEEE 57: $k\in[2,5]$, IEEE 118: $k\in[2,15]$), evaluated by AC power-flow, and only convergent cases are retained. Within each scenario, convergent contingencies are sorted by AC power-flow-severity, and the average severity of the top-$\num$ is computed; curves report the average of this quantity over the 200 scenarios. Comparing panels at the same $\num$ shows that diffusion-based recommendations concentrate on higher-severity (higher-impact) convergent contingencies than uniform random sampling.}

    \label{fig:nk_rec_all}
\end{figure*}

\begin{figure*}[t]
    \centering
    \begin{subfigure}[t]{0.48\textwidth}
        \centering
        \includegraphics[width=\linewidth]{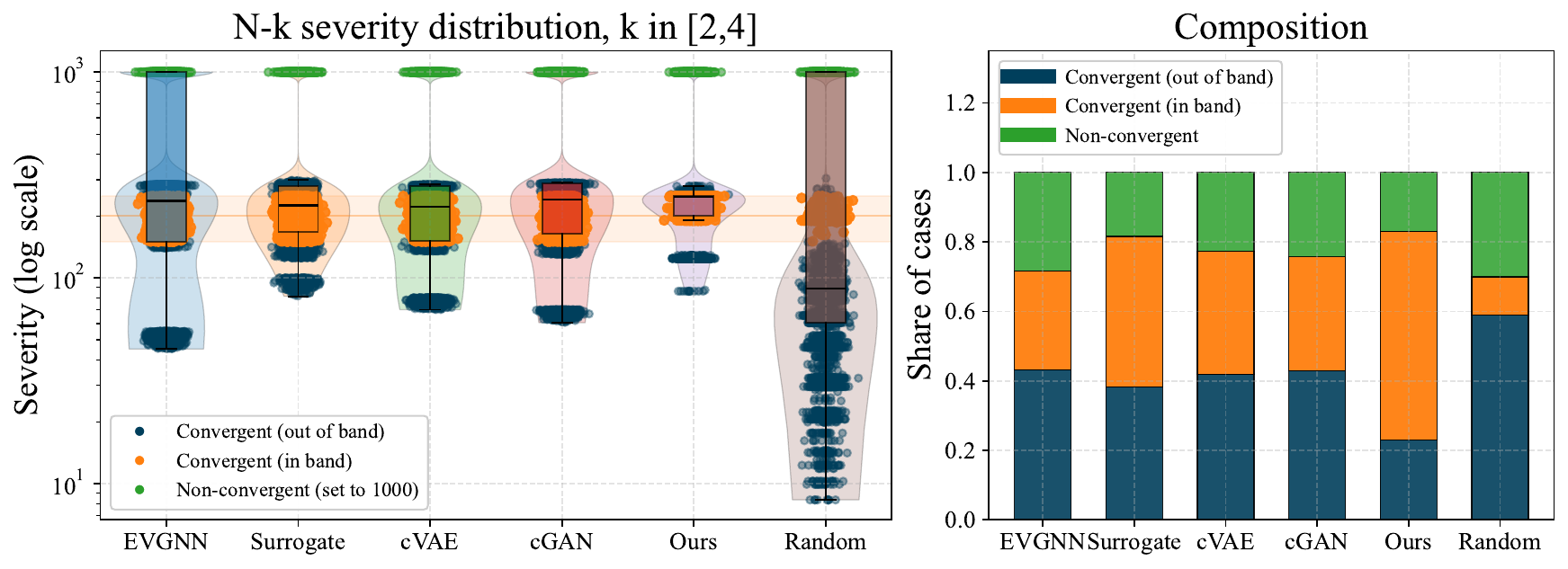}
        \caption{IEEE 14-bus, $k \in [2,4]$.}
        \label{fig:nk_hist_14}
    \end{subfigure}
    \hfill
    \begin{subfigure}[t]{0.48\textwidth}
        \centering
        \includegraphics[width=\linewidth]{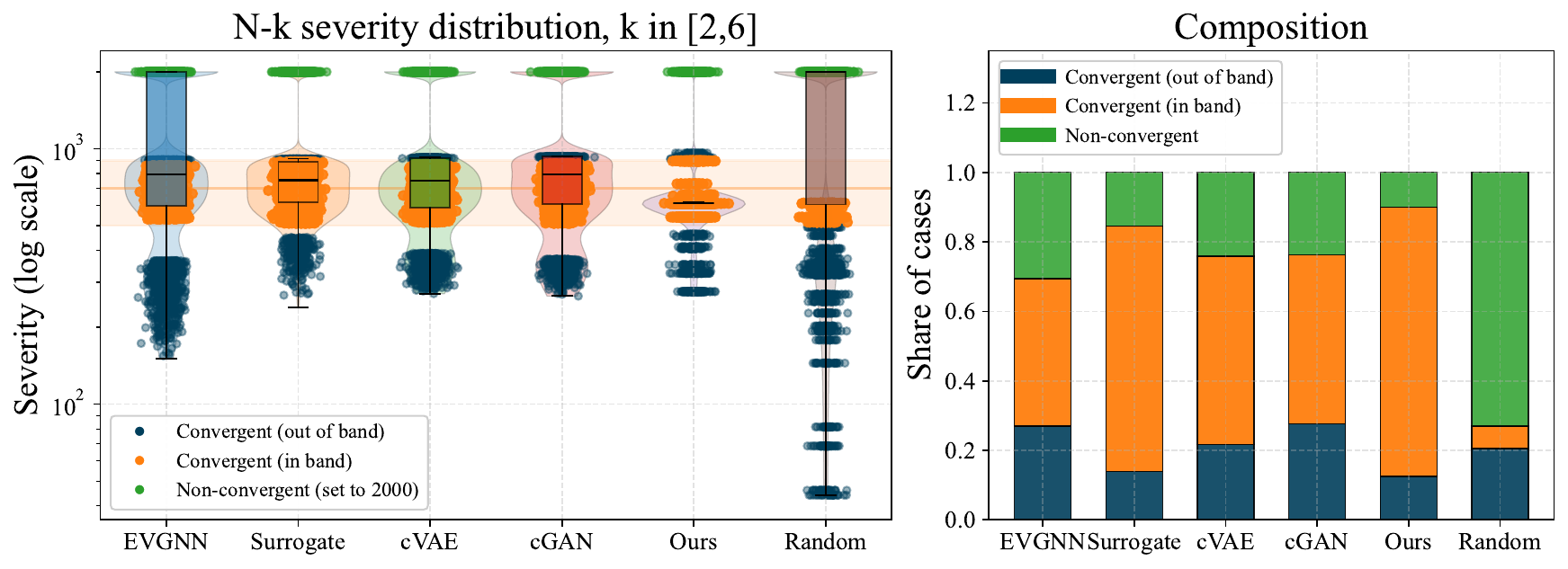}
        \caption{IEEE 39-bus, $k \in [2,6]$.}
        \label{fig:nk_hist_39_comparison}
    \end{subfigure}

    \vspace{0.6em}

    \begin{subfigure}[t]{0.48\textwidth}
        \centering
        \includegraphics[width=\linewidth]{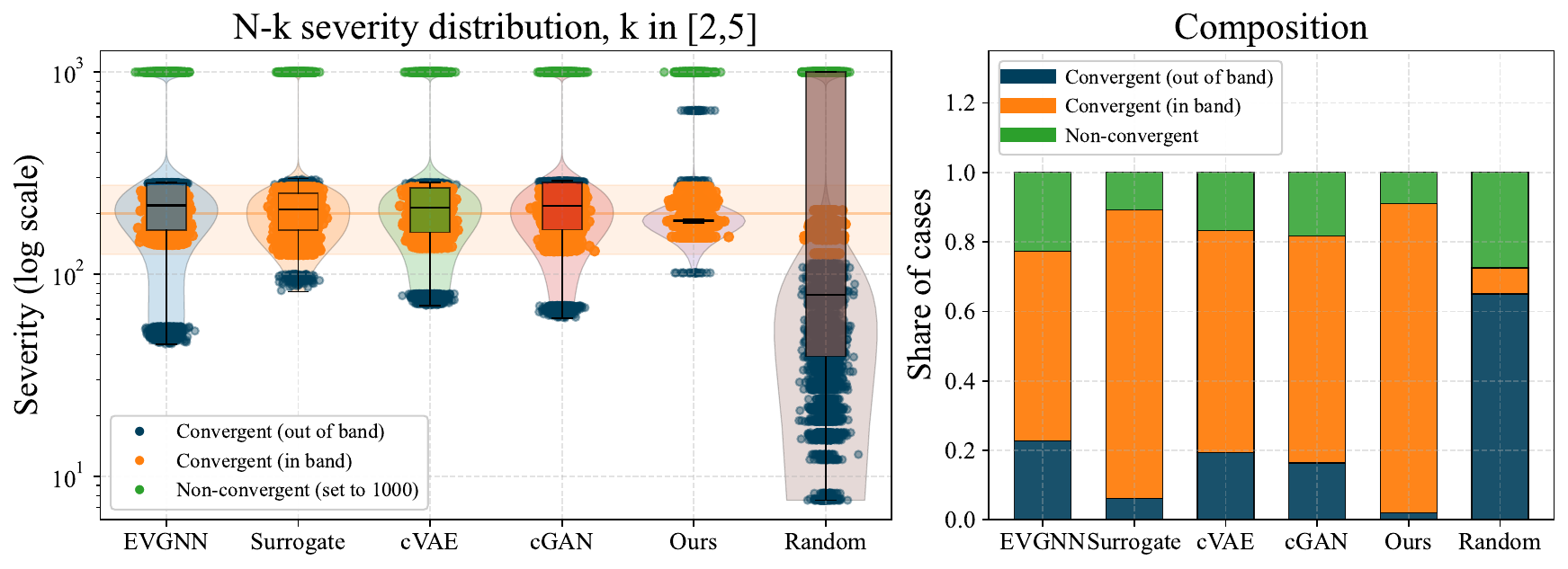}
        \caption{IEEE 57-bus, $k \in [2,5]$.}
        \label{fig:nk_hist_57_comparison}
    \end{subfigure}
    \hfill
    \begin{subfigure}[t]{0.48\textwidth}
        \centering
        \includegraphics[width=\linewidth]{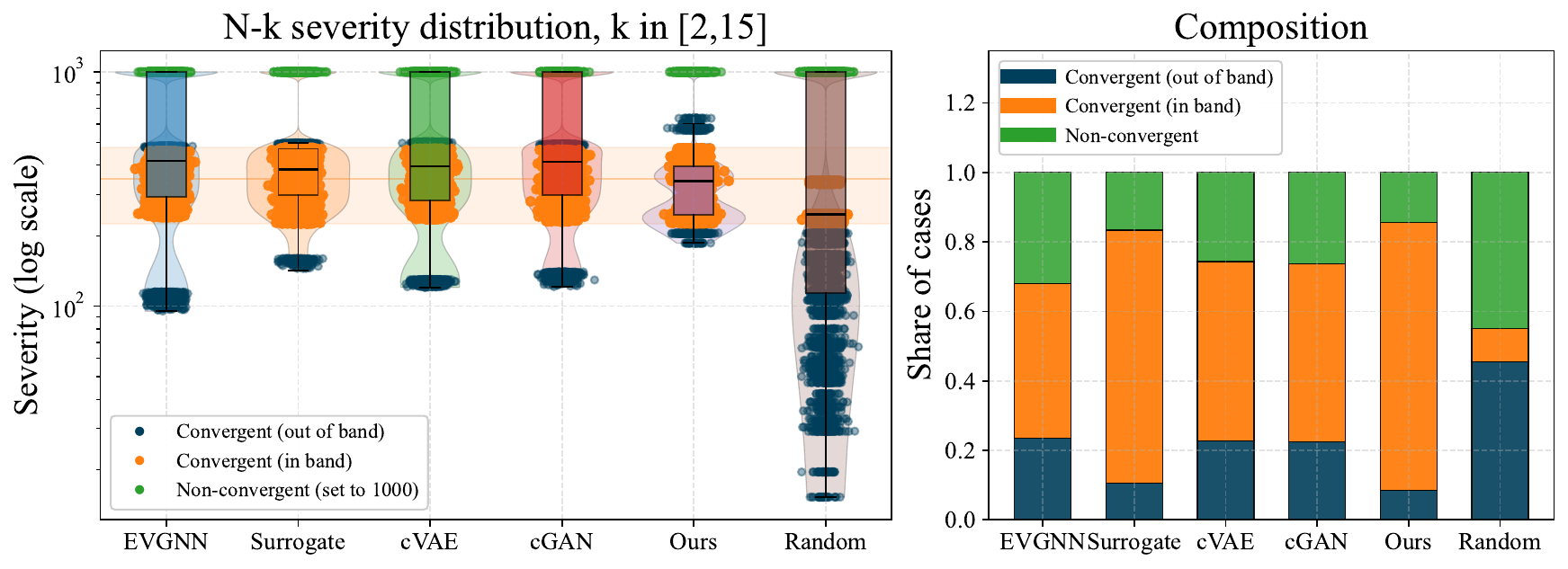}
        \caption{IEEE 118-bus, $k \in [2,15]$.}
        \label{fig:nk_hist_118_comparison}
    \end{subfigure}

    \caption{AC power-flow severity distributions for sampled $N\!-\!k$ contingencies on the IEEE 14-, 39-, 57-, and 118-bus systems with $k \in [2,4]$, $[2,6]$, $[2,5]$, and $[2,15]$, respectively. For each system, each method generates a fixed number of contingencies per operating scenario, which are evaluated using AC power flow. Left panels show severity distributions on a logarithmic scale, with outcomes categorized as ACPF-convergent in-band, ACPF-convergent out-of-band, or ACPF-nonconvergent. Right panels report the corresponding outcome composition as stacked proportions.}

    \label{fig:nk_hist_all_comparison}
\end{figure*}

\section{Experiment}
\label{sec:results}

We evaluate the proposed EVGNN--diffusion screening framework on standard IEEE test systems and assess whether it resolves the two practical limitations identified in Section~\ref{sec:problem}: 
(i) high computational burden due to the need to examine many $N\!-\!k$ outage combinations, and 
(ii) lack of assurance (coverage) in heuristic screening/ranking, especially under stressed nonlinear operating points.
Accordingly, our results focus on two questions: 
(a) sample efficiency under a fixed AC power-flow security-check budget (how much stress we can uncover per AC power-flow check), and 
(b) concentration/coverage of high-severity contingencies (how reliably the screened set contains severe events).

Following Section~\ref{sec:method}, we first fit the EVGNN-based screening index using only base-case and $N\!-\!1$ AC power-flow studies, then use it to select a compact set of high-risk multi-outage cases for labeling and diffusion fitting. 
Unless otherwise stated, AC power flows are solved with PYPOWER/MATPOWER, and severity indices are computed as in Section~\ref{subsec:data}. 
All benchmarks use the same number of AC power-flow checks per operating point to ensure a fair comparison.

\subsection{Experimental setup}

For each IEEE test system (14-, 39-, 57-, and 118-bus), we generate a collection of operating points $\mathbf{x}$ by perturbing active and reactive loads as well as generator setpoints around nominal values. 
Operating points whose base-case AC power-flow does not converge are discarded. 
For each retained operating point $\mathbf{x}_t$, we run AC power-flow for the base case and all single-line outages, yielding the labeled dataset
\[
\mathcal{D}_{N-1}
=
\{(\mathbf{x}_t,\mathbf{c}, s=S(\mathbf{x}_t,\mathbf{c})) : \|\mathbf{c}\|_0=1\}.
\]
The EVGNN screening model $\hat S_\phi(\mathbf{x},\mathbf{c})$ is fitted exclusively on $\mathcal{D}_{N-1}$, with no access to any $N\!-\!k$ ($k\ge 2$) AC labels.
To construct a compact set of high-risk $N\!-\!k$ cases for diffusion fitting, we follow Section~\ref{sec:method}:
for each training operating point $\mathbf{x}_t$, we draw many $k$-outage patterns $\mathbf{c}$ with
\[
k \in [k_{\min},k_{\max}]
=
\begin{cases}
[2,4],  & \text{IEEE 14-bus},\\
[2,6],  & \text{IEEE 39-bus},\\
[2,5],  & \text{IEEE 57-bus},\\
[2,15], & \text{IEEE 118-bus},
\end{cases}
\]
compute the EVGNN screening index $\hat S_\phi(\mathbf{x}_t,\mathbf{c})$, and retain only the most severe (highest-index) cases to form $\mathcal{H}$. 
The state-conditioned diffusion generator $p_\theta(\mathbf{c}\mid\mathbf{x})$ is then fitted on this high-risk set.
During online evaluation, for a new operating point $\mathbf{x}$, the diffusion generator produces a screened set of $\num$ amount of $N\!-\!k$ contingencies
\[
\mathbf{c}^{(1)},\ldots,\mathbf{c}^{(\num)} \sim p_\theta(\cdot\mid\mathbf{x}),
\]
and each is checked using full AC power-flow to obtain its true severity $S(\mathbf{x},\mathbf{c})$.

As benchmarks, we evaluate the proposed EVGNN–diffusion framework against a diverse set of baselines to assess whether it can reliably identify severe $N-k$ contingencies
The baselines include:
(i) EVGNN alone, which serves as an ablation that removes the generative diffusion component and isolates the effect of risk scoring without generative exploration,
(ii) a surrogate-model-based screening approach~\cite{zamzam2020learning}, which replaces AC power-flow evaluation with a learned proxy in conventional $N-k$ analysis,
(iii) conditional generative baselines based on a conditional variational autoencoder (cVAE) and a conditional generative adversarial network (cGAN), which directly compete with the proposed conditional diffusion model, and 
(iv) uniform random sampling of $N-k$ contingencies.

\begin{table}[t]
\centering
\caption{ACPF convergence and in-band rate for Our EVGNN-diffusion-generated $N-k$ contingencies.}
\label{tab:diffusion_convergence}
\setlength{\tabcolsep}{4pt} 
\renewcommand{\arraystretch}{0.95}
\begin{tabular}{lccc}
\toprule
System & $k$ & Conv. (\%) & In-band (\%) \\
\midrule
IEEE 14 & [2,4]  & 85.0  & 84.2 \\
IEEE 39 & [2,6]  & 93.5  & 87.7 \\
IEEE 57 & [2,5]  & 86.5  & 96.5 \\
IEEE 118& [2,15] & 86.0  & 91.9 \\
\midrule
Overall & --     & 87.1  & 90.1 \\
\bottomrule
\end{tabular}
\end{table}

\begin{table}[t]
\renewcommand{\arraystretch}{1.3}
\setlength{\tabcolsep}{4pt}
\centering
\caption{Total runtime (minutes) on the IEEE 39-bus system with $k\in[2,6]$ and $\num=200$. 
}
\label{tab:runtime_39bus_min}
\begin{tabular}{c | c c >{\bfseries}c | c c}
\toprule
$k$
& Exhaustive
& Surrogate-based
& Ours
& Exhau./Ours
& Surro./Ours\\
\midrule
1 & 0.0023 & 0.0006 & 0.001 & $2.3 \times$     & $0.6 \times$ \\
2 & 0.0446 & 0.0130 & 0.002 & $22.3 \times$    & $6.5 \times$ \\
3 & 0.5510 & 0.1604 & 0.004 & $138 \times$     & $40.1 \times$ \\
4 & 4.9594 & 1.4438 & 0.006 & $827 \times$     & $241 \times$ \\
5 & 34.716 & 10.107 & 0.007 & $4,960 \times$    & $1,444 \times$ \\
6 & 196.72 & 57.274 & 0.009 & $21,858 \times$   & $6,364 \times$ \\
\bottomrule
\end{tabular}
\end{table}

\subsection{Top-$\num$ screening curves under a fixed AC power-flow evaluation budget}
\label{subsec:topk}

To evaluate screening efficiency under a limited post-contingency AC power-flow (ACPF) evaluation budget, we report Top-$\num$ curves constructed from ACPF-convergent $N$--$k$ contingencies.
For each operating point, each method produces $\num$ multi-outage patterns in the prescribed $k$-range and runs ACPF for each. 
Only ACPF-convergent cases are retained; within each operating point, these convergent contingencies are ranked by the ACPF severity index $S(\mathbf{x},\mathbf{c})$ in descending order. 
For $\num=1,\dots,200$, we compute the average severity of the top-$\num$ contingencies and then average this quantity over 200 operating points.

Fig.~\ref{fig:randomgenerated} shows the Top-$\num$ curves under uniform random selection, while Fig.~\ref{fig:diffusiongenerated} shows the curves under EVGNN-diffusion screening. 
The key observation is that diffusion screening yields a much slower degradation of the average Top-$\num$ severity as $\num$ increases. 
Under random selection, the first few contingencies can be severe, but the curve drops quickly as more contingencies are included, indicating that additional ACPF evaluations are rapidly diluted by lower-impact outages. 
In contrast, the diffusion curves remain comparatively flat over a wide range of $\num$, indicating that the recommended list continues to contain highly stressed, ACPF-solvable contingencies as the operator expands the screened set.

This effect is consistent across all four systems. 
For IEEE 118-bus, the average Top-$\num$ severity is approximately 430 at $\num=50$ and remains around 310 at $\num=200$ under diffusion screening, whereas random selection is around 250 at $\num=50$ and drops to around 130 at $\num=200$. 
For IEEE 39-bus, diffusion maintains an average Top-$\num$ severity around 720 at $\num=50$ and about 620 at $\num=200$, while random selection drops from about 620 at $\num=50$ to about 400 at $\num=200$. 
For IEEE 57-bus and IEEE 14-bus, diffusion similarly sustains substantially higher average Top-$\num$ severities at moderate budgets (e.g., $\num=50$) and retains a clear margin at $\num=200$.

Table~\ref{tab:runtime_39bus_min}
reports the total runtime required for $N-k$ contingency analysis on the IEEE 39-bus system as $k$ increases. 
The observed runtimes are consistent with the complexity analysis in Section~\ref{subsec:complexity}.
As expected, exhaustive evaluation exhibits rapid combinatorial growth, with runtime increasing by several orders of magnitude as $k$ grows from 1 to 6. Surrogate-based screening substantially reduces the computational burden relative to exhaustive enumeration, but its runtime still scales steeply with $k$ due to the need to evaluate a large number of candidate contingencies. In contrast, the proposed generative approach maintains a nearly constant runtime across all outage orders, reflecting its fixed-budget generative screening strategy. These results highlight the scalability advantage of the proposed method for higher-order contingency analysis, where exhaustive and surrogate-based approaches quickly become impractical.

Section~\ref{subsec:guarantee} (Theorem~\ref{thm:coverage}) formalizes how sampling-and-ranking can retain the high-severity tail with high probability when the sampling distribution places sufficient probability mass in the high-severity region. 
In the experiments, we therefore report two direct ACPF-based indicators of this high-severity concentration: the ACPF severity distribution of evaluated $N-k$ cases and the in-band fraction among ACPF-convergent cases.

For each method, shown in Fig.~\ref{fig:nk_hist_all_comparison}, we compare the resulting ACPF severity distributions and partition outcomes into three categories. These include ACPF-convergent cases outside the high-severity band, ACPF-convergent cases inside the high-severity band, and ACPF-nonconvergent cases, with nonconvergence mapped to large sentinel severities as specified in the figure caption. Across all IEEE test systems, the proposed diffusion-based approach consistently concentrates a larger fraction of samples within the high-severity in-band region while maintaining a favorable ACPF convergence profile. In contrast, EVGNN alone exhibits limited exploration of the severe tail, surrogate screening shows weaker severity concentration, and cVAE and cGAN baselines demonstrate less consistent focusing on the target high-impact region. Relative to uniform random sampling, the diffusion model shifts probability mass from low-severity events toward the high-severity tail and produces a substantially larger share of severe yet ACPF-convergent contingencies without inducing excessive nonconvergent outcomes. The composition bars further illustrate that these gains arise from the combined effect of topology-aware risk scoring and diffusion-based generation rather than risk scoring alone.

\subsection{High-severity retention and in-band rates under diffusion screening}
\label{subsec:nk_dist}

To provide a clearer comparison between uniform random sampling and diffusion-based screening, Fig.~\ref{fig:nk_hist_14} presents a magnified view of the IEEE 14-bus case with $k\in[2,4]$ extracted from Fig.~\ref{fig:nk_hist_all_comparison}. 
In the zoomed plot, uniform random sampling produces a broad spread of ACPF-convergent contingencies across the severity range, with a substantial portion lying outside the designated high-severity band. 
By contrast, diffusion-based screening concentrates ACPF-convergent contingencies within the high-severity band, and the central tendency of the convergent samples is shifted upward relative to random sampling. 

Table~\ref{tab:diffusion_convergence} quantifies the same effect for diffusion screening. 
The ACPF convergence rates are 85.0\% (IEEE 14), 93.5\% (IEEE 39), 86.5\% (IEEE 57), and 86.0\% (IEEE 118), and among the ACPF-convergent cases the in-band fractions are 84.2\%, 87.7\%, 96.5\%, and 91.9\%, respectively (overall: 87.1\% convergent and 90.1\% in band). 
These values indicate that diffusion screening produces a screened set in which most ACPF-solvable contingencies fall inside the high-severity band, which is the empirical condition required for the sampling-and-ranking guarantee in Theorem~\ref{thm:coverage} to be operationally useful when selecting how many contingencies to generate and retain low ACPF/ACOPF runtime.

\begin{figure}[t!]
        \centering
        \includegraphics[width=\linewidth]{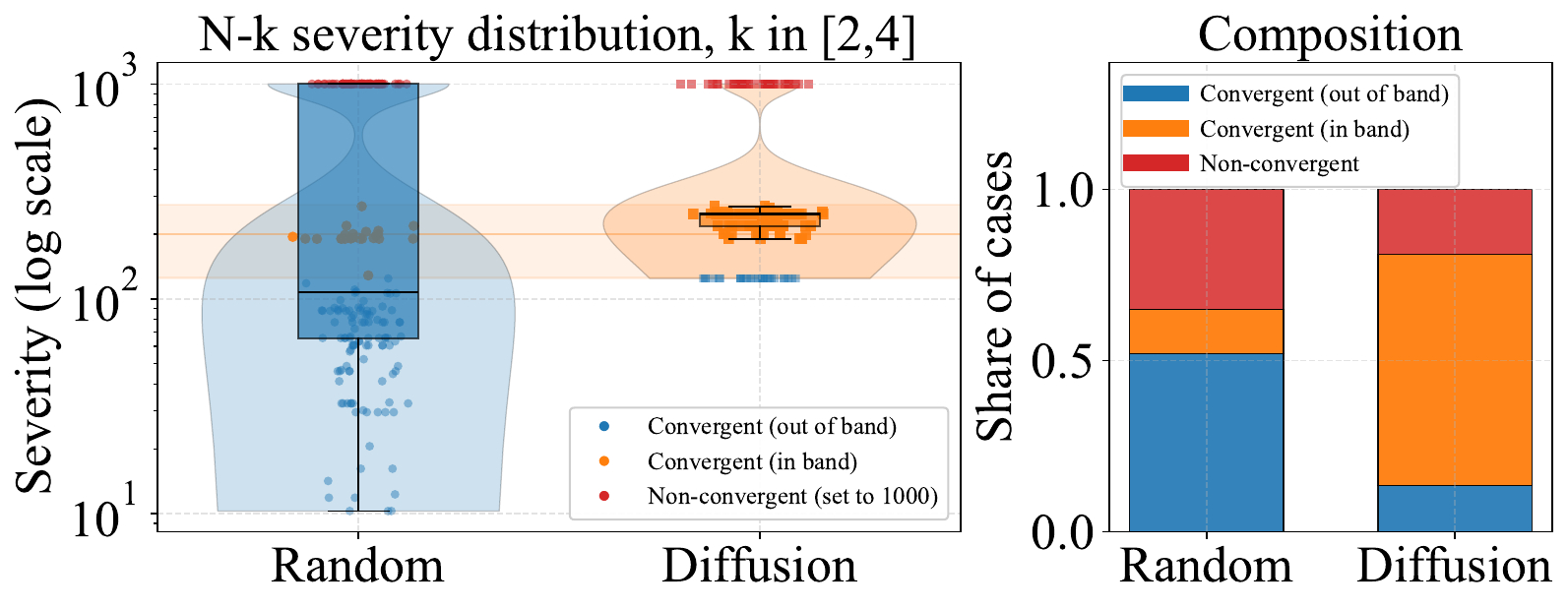}
        \caption{Detailed view of the IEEE 14‑bus results from two methods from the four‑system composite figure (14/39/57/118‑bus). ACPF severity distribution for sampled $N\!-\!k$ contingencies and corresponding outcome composition.}
        \label{fig:nk_hist_14}
        \vspace{-5mm}
\end{figure}

\section{Conclusion and Future Work}
\label{sec:conclusion}

This paper formulates $N\!-\!k$ security assessment as a state-conditioned contingency scenario generation problem: for a given operating point $\mathbf{x}$, the objective is to produce a small screened set of multi-element outage patterns $\mathbf{c}$ without enumerating the combinatorial $N\!-\!k$ space. 
We develop an EVGNN--diffusion workflow in which an EVGNN-based screening index, trained only on base-case and $N\!-\!1$ AC power-flow studies, is used to construct a compact high-risk set $\mathcal{H}$ for diffusion fitting, and a state-conditioned diffusion generator $p_{\theta}(\mathbf{c}\mid\mathbf{x})$ is trained with severity-aware weighting to emphasize high-severity events. 
Simulation results show that the proposed method identifies higher-severity contingencies and enables operators to prioritize computational resources on the most critical outage scenarios under limited validation budgets.
It supports coverage-controlled guarantees based on a user-specified tolerance for missing severe contingencies.

\bibliographystyle{IEEEtran}
\bibliography{IEEEabrv,reference}

\end{document}